\documentclass[conference]{IEEEtran}
\IEEEoverridecommandlockouts
\usepackage{cite}
\usepackage{amsmath,amssymb,amsfonts}
\usepackage{algorithmic}
\usepackage{graphicx}
\usepackage{textcomp}
\usepackage{xcolor}
\usepackage{booktabs}
\usepackage{tabularx}
\usepackage{stfloats}
\graphicspath{{escape_aware_cbf/}}
\usepackage{amsthm}
\newtheorem{lemma}{Lemma}
\newtheorem{proposition}{Proposition}
\newtheorem{theorem}{Theorem}
\newtheorem{corollary}{Corollary}
\newtheorem{remark}{Remark}
\newtheorem{assumption}{Assumption}
\def\BibTeX{{\rm B\kern-.05em{\sc i\kern-.025em b}\kern-.08em
    T\kern-.1667em\lower.7ex\hbox{E}\kern-.125emX}}
\usepackage{caption}
\usepackage{subcaption}

\begin{document}

\title{Escape-Aware Control Barrier Functions for\\
Quadrotor Safety under Body-Rate Limits
}

\author{
\IEEEauthorblockN{Lei Shi\textsuperscript{1}, Haosong Wen\textsuperscript{2}, Qichao Liu\textsuperscript{3,*}}
\thanks{\textsuperscript{1}University of Wisconsin--Madison.
\textsuperscript{2}Southeast University.
\textsuperscript{3}The Hong Kong Polytechnic University.
\textsuperscript{*}Corresponding author: Qichao Liu.}
}

\maketitle

\begin{abstract}
Control barrier functions for input-constrained systems place the admissible
input set inside the definition of the safe set, yet the resulting barrier is
almost always a function of the state alone. On a quadrotor this is not
cosmetic: because the thrust vector must be reoriented before it can decelerate
an approach, and reorientation is limited by the attainable body rate, a
state-only barrier certifies states from which no escape is reachable in time.
We characterize the certification gap in closed form and show its width is
proportional to closing speed and inversely proportional to the body-rate
limit. We then define an escape barrier on the augmented pair of state and
previously applied input, with escape authority measured over the one-step
reachable thrust cap. It admits a closed form and an analytic inverse for the
maximum certifiable closing speed, and embeds in a predictive controller at no
additional state cost. Across $550$ paired closed-loop episodes on a $13$-state quadrotor, the
proposed controller completes every tested scenario, whereas the
stopping-distance barrier enforced over the same horizon fails $15\%$ and
$25\%$ of episodes in exactly the two scenarios that enter the predicted gap.
Against an online backup-CBF baseline enforcing the same escape condition at
the reached state, it holds a $29$--$74^\circ$ larger directional margin and
$3$--$18$ times the clearance, and an independent conservative rollout referee
finds no certified state from which escape fails.

\end{abstract}


\section{Introduction}

Many control barrier function (CBF) based methods formulate quadrotor safety
as geometric clearance: remain outside a prescribed margin around obstacles.
This is inexpensive, combines readily with model predictive control (MPC), and
extends to multiple obstacles. But clearance alone does not characterize
safety, because remaining safe also depends on whether the thrust direction can
be reoriented toward the escape direction in time. A quadrotor maneuvers by
changing the magnitude and direction of its thrust, and direction is set by
attitude: magnitude changes quickly, direction must be slewed and is limited by
the attainable body rate. The vehicle therefore cannot brake instantaneously
along the obstacle normal --- it must first rotate the thrust, and clearance is
consumed throughout (Fig.~\ref{fig:gap}). The question is not how much
clearance remains, but whether it suffices for an escape feasible under the
reachable thrust-direction constraint.

Escape-based methods decide safety by the existence of a feasible escape
rather than instantaneous distance. Inevitable collision
states~\cite{fraichard2003}, backup CBFs~\cite{gurriet2018,chen2021},
Hamilton--Jacobi reachability~\cite{choi2021}, and closed-form
input-constrained barriers~\cite{breeden2023,wang2017} all fold control
authority into certification, but to stay tractable for feedback they typically
collapse the input constraint into a state-only barrier independent of $u^-$.
That is reasonable for point-mass models, where the braking direction is
history-independent; for a quadrotor the current thrust direction decides
whether the escape direction is reachable in time. A set of states then
satisfies the stopping-distance barrier $h_I(x)\geq 0$ yet admits no input that
reorients and brakes before collision under a bounded body rate. We
therefore propose an escape-aware CBF on the augmented pair $(x,u^-)$, built
from the one-step reachable thrust set induced by the previously applied input,
which splits an escape into reorientation followed by aligned braking.

\begin{figure}[t]
    \centering
    \includegraphics[width=0.55\linewidth]{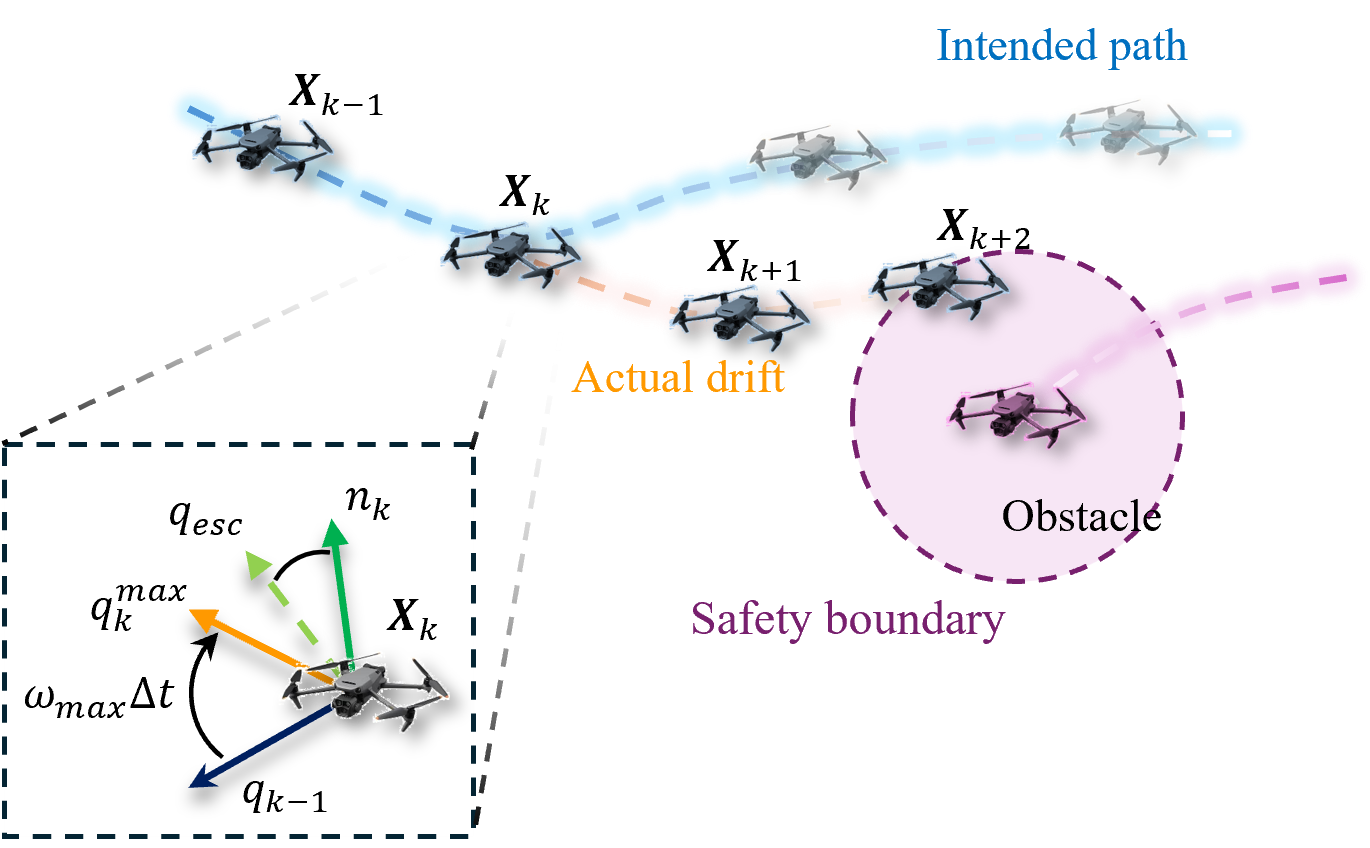}
\vspace{-2mm}
    \caption{Thrust-reorientation drift. The intended path is geometrically safe,
    but the thrust cannot align with the escape direction within one control
    interval, so the trajectory drifts toward the boundary before braking is
    available.}
    \label{fig:gap}
\vspace{-5mm}
\end{figure}

The main contributions of this paper are summarized as:
\begin{enumerate}
\item \textbf{A closed-form certification gap.} We prove that the stopping-distance barrier $h_I$, equivalently the backup CBF
induced by a maximum-braking policy~\cite{chen2021}, certifies a nonempty set
of states from which no body-rate-feasible escape exists, with gap width
$\Delta d_{\mathrm{gap}}(\phi)=\max\{0,\phi-\pi/2\}\,v_{\mathrm{cl}}/\omega_{\max}$.
It grows linearly in closing speed, vanishes only as
$\omega_{\max}\!\to\!\infty$, and is invisible to any barrier that is a
function of $\mathbf{x}$ alone, so no constant margin absorbs it.

\item \textbf{An escape certificate on $(\mathbf{x},u^-)$ with a closed form
and an analytic inverse.} Built from the one-step reachable thrust cap induced
by the previously applied input, it assigns different safety values to states
with identical position, velocity, and clearance but different thrust
directions --- a distinction no state-only barrier expresses. Unlike backup and
implicit CBFs, which roll out a backup policy numerically, ours is algebraic:
$h_{\mathrm{esc}}$ in closed form, an analytic inverse for the maximum
certifiable closing speed, a proof that the certified set is contained in the escapable set under
the stated reorientation and drift assumptions (Theorem~\ref{thm:conservatism}), and exact reduction to the isotropic braking bound as $\theta\!\to\!\pi$.

\item \textbf{Zero-cost embedding in MPC, validated against a backup-CBF
oracle.} Because MPC already optimizes an input sequence, the dependence on
$u^-$ propagates along the horizon with no state augmentation. On a 13-state
plant with motor lag and a cascaded attitude loop, the controller holds
significantly larger clearance and directional margin than a geometric
MPC-CBF~\cite{zeng2021}, stays feasible where a constant-margin
filter~\cite{mandaokar2026} is infeasible from the first step, and beats a
rate-feasible backup CBF~\cite{gurriet2018,chen2021,singletary2022} flown
online. Internal ablations isolate two variables: $h_I$ imposed over the same
horizon still fails $15$--$25\%$ of episodes in the two scenarios that enter
the gap, and our own certificate imposed at one node instead of $20$ loses the
margin. The
enforcement locus is what the margin depends on, and the closed form is what
makes that locus affordable.
\end{enumerate}

\section{Related Work}
\label{sec:related}
\subsection{Escape-Based Safety Certificates}

Escape-based certificates define safety through the existence of future
admissible motions rather than instantaneous clearance. Fraichard and
Asama~\cite{fraichard2003} introduced inevitable collision states.
\emph{Difference:} it is a set-membership characterization, providing no scalar
barrier, Lie-derivative condition, or forward-invariance certificate for online
control.

A second line places actuator limits algebraically inside the certificate.
Wang~\emph{et al.}~\cite{wang2017} couple pairwise distance and relative
velocity through a maximum braking capability. Breeden and
Panagou~\cite{breeden2023} give the most developed form, with a certified set
defined through an explicitly parameterized evading maneuver at the input
bound. \emph{Difference:} that maneuver is charged against a bound on input
\emph{magnitude} and assumed instantaneously available at every state, its
admissibility independent of the previously applied input. We charge the
reorientation \emph{time} needed to make the braking direction admissible, and
our certified set is a function of $(\mathbf{x},u^-)$: for a quadrotor the
limiting factor is not the magnitude of available braking but its direction, so
two states with identical position, velocity, and clearance can differ in
escape capability.

\subsection{Backup and Implicit Safety Certificates}

Backup and implicit CBFs certify states whose finite-time rollout under a
prescribed backup policy stays inside the constraint set;
Chen~\emph{et al.}~\cite{chen2021} observe that the stopping-distance barrier
is the backup CBF induced by a maximum-braking policy. Singletary~\emph{et al.}~\cite{singletary2022} run backup CBFs onboard racing
quadrotors under thrust and attitude-rate limits --- the same actuation limits
we study. \emph{Difference:} their certificate is evaluated by numerically
integrating a fixed backup controller from the current \emph{state}, returning
a scalar verdict rather than a differentiable function of the control decision,
so it can filter a candidate input but cannot shape an optimization over a
horizon. We derive the escape distance algebraically on $(\mathbf{x},u^-)$, so
the identical condition can be imposed at every prediction node. The
distinction is not cost --- Sec.~\ref{subsec:closed_loop} measures the
backup-CBF controller as no slower than ours --- but that a one-step filter
rejects an input only once the state requiring rejection is reached. We use
such a rollout as the conservative referee of Sec.~\ref{subsec:baselines}, fly
it as B3, and move our own certificate to the same single-node locus in
Sec.~\ref{subsec:ablation}. Projection-based
extensions handle mixed state--input constraints~\cite{gacsi2026}.
\emph{Difference:} \cite{gacsi2026} composes backup policies over constraints
on the instantaneous pair $(\mathbf{x},u)$ and does not bound how far $u$ may
move from $u^-$ in one interval --- precisely the constraint generating our
gap.

A third route computes the escapable set directly: Hamilton--Jacobi
reachability yields the exact viability kernel for a given input set, and
\cite{choi2021,tonkens2022} use the value function as, or to refine, a CBF.
\emph{Difference:} the offline PDE is solved over a fixed state space, so the
input set must depend on the state alone; a $u^-$-dependent
$\mathcal{U}_1(u^-)$ would lift the grid from six to nine dimensions. Ours is
not a tighter set, but a closed-form certificate for that set.

\subsection{Margin-Based Safety for Aerial Vehicles}

Aerial avoidance must account for latency, actuator limits, obstacle motion,
and estimation uncertainty, so many methods inflate obstacles. The
distributionally robust acceleration barrier filter of~\cite{mandaokar2026}
folds latency, actuator and jerk limits, obstacle acceleration, and perception
uncertainty into an effective clearance. \emph{Difference:} such margins are set by worst-case
constants and are direction-independent. We compute a direction-dependent
escape authority from the one-step reachable thrust set induced by $u^-$.

Two lines relax that constant. Liu and Wang~\cite{liu2026adaptive} size an
MPC-embedded CBF to the observable region. \emph{Difference:} their margin
adapts to what is \emph{perceived}, ours to what is \emph{actuable} from $u^-$.
Shi~\emph{et al.}~\cite{shi2024fairness} enrich the barrier with obstacle
shape, velocity, and orientation, and distribute control authority across
robots. \emph{Difference:} there authority is scheduled between agents; here it
is thrust authority reachable within one interval on one vehicle.

\section{Problem Formulation}
\label{sec:problem}
\subsection{Quadrotor Dynamics and Input Reachability}
\label{subsec:dynamics_input_reachability}

With $\mathbf{p},\mathbf{v}\in\mathbb{R}^3$ the position and velocity,
$\mathbf{x}:=(\mathbf{p},\mathbf{v})$, and $\mathbf{e}_3=[0,0,1]^\top$, the
control-oriented translational dynamics are
\begin{equation}
\dot{\mathbf{p}} = \mathbf{v},\qquad
\dot{\mathbf{v}} = m^{-1} f\mathbf{q} - g\mathbf{e}_3 + \mathbf{w},
\label{eq:quad_dynamics}
\end{equation}
with $m$ the mass, $g$ gravity, and $\mathbf{w}\in\mathcal{W}$ a bounded
disturbance. The input is
\begin{equation}
u := (f,\mathbf{q})\in [f_{\min},f_{\max}]\times \mathbb{S}^2,
\label{eq:input_def}
\end{equation}
with $f$ the thrust magnitude and $\mathbf{q}$ its direction.

The thrust direction is set by attitude and therefore cannot change
arbitrarily between control intervals. With $u^-=(f^-,\mathbf{q}^-)$ the
previously applied input, thrust-rate and body-rate limits give the one-step
reachable input set
\begin{equation}
\mathcal{U}_1(u^-)
=
\left\{
(f,\mathbf{q}) :
\begin{aligned}
& |f-f^-| \leq \dot f_{\max}\Delta t,\quad
f\in[f_{\min},f_{\max}],\\
& \angle(\mathbf{q},\mathbf{q}^-)\leq \theta
\end{aligned}
\right\},
\label{eq:one_step_input_set}
\end{equation}
where $\Delta t$ is the outer sampling period, $\dot f_{\max}$ the
thrust-rate bound, and $\theta$ the maximum thrust-direction change achievable
in one interval. This is not an isotropic acceleration ball but a cap oriented
by $\mathbf{q}^-$. The angle $\theta$ is not $\omega_{\max}\Delta t$: the
attitude response has a finite lag, so
$\theta \leq \omega_{\max}(\Delta t-t_{\mathrm{lag}})
< \omega_{\max}\Delta t$,
with $\omega_{\max}$ the sustained slew rate and $t_{\mathrm{lag}}$ the
closed-loop attitude-response lag. The analysis uses only the following
identified property.

\begin{assumption}[Identified reorientation]
\label{asm:reorient}
From any admissible attitude, the thrust axis can be rotated onto a target
direction $\psi$ radians away within
$t_{\mathrm{lag}}
+
{\max\{0,\psi-\theta\}}/{\omega_{\max}}$.
\end{assumption}

For each obstacle $\mathcal{O}_i$ let $d_i(\mathbf{p})$ be the signed distance
to its boundary, so the free distance is
$d_{\mathrm{free}}(\mathbf{p})
=
\min_i d_i(\mathbf{p})$,
and the corresponding most critical obstacle is
$i^\star(\mathbf{p})
=
\arg\min_i d_i(\mathbf{p})$.

Let $\mathbf{n}(\mathbf{p})$ denote the outward unit normal at the boundary
of $\mathcal{O}_{i^\star}$, and let $\mathbf{v}_o$ be the velocity of that
obstacle at its closest point. All escape quantities are evaluated in the frame
of the closest obstacle, so the closing speed is
\begin{equation}
v_{\mathrm{cl}}(\mathbf{x})
=
\max\{0,-\mathbf{n}^{\top}(\mathbf{p})(\mathbf{v}-\mathbf{v}_o)\},
\label{eq:closing_speed}
\end{equation}
and the thrust-direction misalignment relative to the escape direction is
$\phi
=
\angle(\mathbf{n}(\mathbf{p}),\mathbf{q}^-)$.

\begin{remark}[Moving obstacles]
\label{rem:moving_obstacle}
Eq.~\eqref{eq:closing_speed} is exact for an obstacle at constant velocity. An
accelerating obstacle with $\|\mathbf{a}_o\|\leq\bar a_o$ enters the relative
dynamics along $\mathbf{n}$ exactly as an inward disturbance, so replacing
$\mathcal{W}$ by $\mathcal{W}\oplus\{\mathbf{a}:\|\mathbf{a}\|\leq\bar a_o\}$
inflates $a_{\mathrm{dist}}^{\mathrm{bound}}$ by at most $\bar a_o$ and every
statement below, Theorem~\ref{thm:conservatism} included, holds verbatim. What
is \emph{not} absorbed is rotation of $\mathbf{n}$ itself.
\end{remark}

Two states with identical $\mathbf{p}$, $\mathbf{v}$, and $d_{\mathrm{free}}$
can therefore have different escape authority when $\mathbf{q}^-$ differs,
which is the basis for the gap analyzed next.

\subsection{State-Only Stopping Barrier}
\label{subsec:state_only_stopping_barrier}

An isotropic braking abstraction yields the state-only stopping barrier
\begin{equation}
h_I(\mathbf{x})
=
d_{\mathrm{free}}(\mathbf{p})
-
(2a_{\max})^{-1}v_{\mathrm{cl}}^2(\mathbf{x}).
\label{eq:state_only_barrier}
\end{equation}
with $a_{\max}:=f_{\max}/m$, certifying a state when the free distance exceeds
the distance needed to dissipate the closing speed under maximum braking.

As Chen~\emph{et al.}~\cite{chen2021} note, $h_I$ is the backup CBF induced
by a maximum-braking policy, so it is a genuine state-only certificate under an
isotropic braking abstraction rather than a heuristic margin. Its limitation is
that it does not depend on $u^-$, and therefore assigns the same value to
states with identical position, velocity, and free distance but different
thrust directions --- equivalently, it assumes the braking direction is
instantaneously available. Under Eq.~\eqref{eq:one_step_input_set} that fails.

\subsection{State-Only Certification Gap under Body-Rate Limits}
\label{subsec:gap}

Eq.~\eqref{eq:state_only_barrier} assumes $a_{\max}$ is available immediately
along the escape direction, which fails under
Eq.~\eqref{eq:one_step_input_set}. Consider a horizontal encounter with
$\mathbf{n}$ fixed, no gravity projection along $\mathbf{n}$, and no
disturbance, and let $\psi(t):=\angle(\mathbf{q}(t),\mathbf{n})$. Under the
body-rate limit every feasible thrust-direction trajectory satisfies
\begin{equation}
\psi(t)
\geq
\max\{0,\phi-\omega_{\max}t\},
\label{eq:domination}
\end{equation}
with $\phi$ the initial misalignment. While $\psi(t)>\pi/2$ the thrust has no
positive projection along the escape direction and cannot reduce the closing
speed.

\begin{proposition}[Certification gap]
\label{prop:gap}
In the horizontal encounter above, suppose that the current thrust
direction is opposite to the escape direction, i.e., $\phi=\pi$. Then every
state satisfying
\begin{equation}
\frac{v_{\mathrm{cl}}^2}{2a_{\max}}
\leq
d_{\mathrm{free}}
<
\frac{v_{\mathrm{cl}}^2}{2a_{\max}}
+
\frac{\pi v_{\mathrm{cl}}}{2\omega_{\max}}
\label{eq:gap}
\end{equation}
is certified as safe by the state-only stopping barrier, but admits no
feasible escape maneuver before the available clearance is consumed under
the body-rate limit.
\end{proposition}

\begin{proof}
For $\phi=\pi$, Eq.~\eqref{eq:domination} gives $\psi(t)>\pi/2$ for
$t<\pi/(2\omega_{\max})$, so the thrust has no positive projection along
$\mathbf{n}$ and the closing speed cannot be reduced; at least
$\pi v_{\mathrm{cl}}/(2\omega_{\max})$ of clearance is consumed before braking
begins. The lower bound in Eq.~\eqref{eq:gap} gives $h_I\geq 0$, the upper
bound makes the available clearance smaller than the stopping distance plus the
reorientation distance, so no escape completes before the clearance is
exhausted.
\end{proof}

The gap width follows from Eq.~\eqref{eq:gap}:
\begin{equation}
\Delta d_{\mathrm{gap}}
=
\left(
\frac{v_{\mathrm{cl}}^2}{2a_{\max}}
+
\frac{\pi v_{\mathrm{cl}}}{2\omega_{\max}}
\right)
-
\frac{v_{\mathrm{cl}}^2}{2a_{\max}}
=
\frac{\pi v_{\mathrm{cl}}}{2\omega_{\max}}.
\label{eq:gap_width}
\end{equation}

For fixed $v_{\mathrm{cl}}>0$ this is proportional to $1/\omega_{\max}$ and
vanishes as $\omega_{\max}\to\infty$. It is not small: at the identified
$\omega_{\max}=3.05~\mathrm{rad/s}$ the over-certified band is
$0.51\,v_{\mathrm{cl}}$ wide, $2.6~\mathrm{m}$ at $5~\mathrm{m/s}$. The same
argument applies at any
misalignment, which makes the gap a set of positive measure rather than one
configuration.

\begin{corollary}[Gap at arbitrary misalignment]
\label{cor:gap_phi}
In the horizontal encounter of Proposition~\ref{prop:gap}, let the initial
thrust misalignment be $\phi\in[0,\pi]$. Every state satisfying
\begin{equation}
\begin{split}
\frac{v_{\mathrm{cl}}^2}{2a_{\max}} \leq d_{\mathrm{free}} &< \frac{v_{\mathrm{cl}}^2}{2a_{\max}} + \Delta d_{\mathrm{gap}}(\phi), \\
\Delta d_{\mathrm{gap}}(\phi) &= \frac{\max\{0,\phi-\tfrac{\pi}{2}\}\,v_{\mathrm{cl}}}{\omega_{\max}}
\end{split}
\label{eq:gap_phi}
\end{equation}
satisfies $h_I(\mathbf{x})\geq 0$ but admits no body-rate-feasible escape.
\end{corollary}

\begin{proof}
By Eq.~\eqref{eq:domination}, $\psi(t)>\pi/2$ for
$t<(\phi-\pi/2)/\omega_{\max}$, so at least
$v_{\mathrm{cl}}(\phi-\pi/2)/\omega_{\max}$ of clearance is consumed before
braking begins, and Proposition~\ref{prop:gap} applies with
$\pi/(2\omega_{\max})$ replaced by $(\phi-\pi/2)/\omega_{\max}$. For
$\phi\leq\pi/2$ the interval is empty.
\end{proof}

Two features are important. First, the gap is zero for $\phi\leq\pi/2$ and
then grows linearly, so the state-only barrier is exact only when the thrust
already points outward; this direction-dependent degradation cannot be
reproduced by any $\phi$-independent margin. Second,
Eq.~\eqref{eq:gap_phi} is a lower bound: it charges only the sustained slew
and ignores both $t_{\mathrm{lag}}$ and the inward acceleration produced while
$\psi>\pi/2$. Quadrotor certification therefore cannot depend on $\mathbf{x}$ alone, but
must account for the thrust-direction reachability induced by $u^-$.
\section{Escape-Aware Barrier Construction}
\label{sec:method}
\subsection{One-Step Reachable Thrust Set}
\label{subsec:cap}

The certificate must therefore retain thrust-direction reachability. We first
characterize the accelerations generatable within one interval from $u^-$:
\begin{equation}
\mathcal{A}_1(u^-)
=
\left\{
\frac{f}{m}\mathbf{q}-g\mathbf{e}_3:
f\in[f'_{\min},f'_{\max}],\;
\angle(\mathbf{q},\mathbf{q}^-)\leq\theta
\right\},
\label{eq:reachable_acc}
\end{equation}
where the one-step reachable thrust-magnitude interval is
\begin{equation}
\begin{aligned}
f'_{\max} &= \min\{f_{\max},\,f^-+\dot f_{\max}\Delta t\},\\
f'_{\min} &= \max\{f_{\min},\,f^--\dot f_{\max}\Delta t\}.
\end{aligned}
\label{eq:thrust_window}
\end{equation}

$\mathcal{A}_1(u^-)$ is anisotropic and oriented by $\mathbf{q}^-$. The barrier
below inherits this structure, and with it the ability to separate states that
differ only in thrust orientation.

\subsection{Directional Escape Authority}
\label{subsec:authority}

The reachable set in Eq.~\eqref{eq:reachable_acc} bounds the acceleration
available along the escape direction. We define the one-step escape authority
as the largest outward normal acceleration reachable from $u^-$:
\begin{equation}
a_{\mathrm{esc}}^{\mathrm{avail}}(\mathbf{x},u^-)
=
\max_{\mathbf{a}\in\mathcal{A}_1(u^-)}
\mathbf{n}^{\top}(\mathbf{p})\mathbf{a}.
\label{eq:escape_authority_def}
\end{equation}
Because the input set separates the magnitude interval from the direction cap,
Eq.~\eqref{eq:escape_authority_def} evaluates in closed form. For fixed $f$ the
optimal direction rotates $\mathbf{q}^-$ toward $\mathbf{n}(\mathbf{p})$ as far
as the angular budget allows:
\begin{equation}
\max_{\angle(\mathbf{q},\mathbf{q}^-)\leq\theta}
\mathbf{n}^{\top}(\mathbf{p})\mathbf{q}
=
c(\phi),
c(\phi)
:=
\cos\left(\max\{0,\phi-\theta\}\right),
\label{eq:direction_factor}
\end{equation}
where $\phi=\angle(\mathbf{n}(\mathbf{p}),\mathbf{q}^-)$ is the current
thrust-direction misalignment.

The remaining scalar problem is linear in $f$, so the optimum is at an
endpoint:
\begin{equation}
a_{\mathrm{esc}}^{\mathrm{avail}}(\mathbf{x},u^-)
=
\frac{c(\phi)}{m}
\begin{cases}
f'_{\max}, & c(\phi)\geq 0,\\
f'_{\min}, & c(\phi)<0,
\end{cases}
-
g\,\mathbf{n}^{\top}(\mathbf{p})\mathbf{e}_3 .
\label{eq:escape_authority}
\end{equation}
Three properties follow: it depends explicitly on $\phi$, hence on
$\mathbf{q}^-$; when $\phi-\theta>\pi/2$ no reachable direction has positive
projection along the escape direction, the regime of
Proposition~\ref{prop:gap}; and as $\theta\to\pi$ we recover the isotropic
abstraction. Subtracting the worst inward disturbance projection,
\begin{equation}
a_{\mathrm{dist}}^{\mathrm{bound}}(\mathbf{x})
=
\max_{\mathbf{w}\in\mathcal{W}}
\left(-\mathbf{n}^{\top}(\mathbf{p})\mathbf{w}\right).
\label{eq:disturbance_projection}
\end{equation}
The resulting net escape capability is
\begin{equation}
a_{\mathrm{eff}}(\mathbf{x},u^-)
=
a_{\mathrm{esc}}^{\mathrm{avail}}(\mathbf{x},u^-)
-
a_{\mathrm{dist}}^{\mathrm{bound}}(\mathbf{x}).
\label{eq:effective_escape_acc}
\end{equation}
This is the certified outward acceleration after the worst admissible
disturbance. The certificate of Sec.~\ref{subsec:barrier} is built from
$a_{\mathrm{eff}}$: both quantities appearing there are instances of
Eq.~\eqref{eq:effective_escape_acc} at different misalignments. Writing
$a_{\mathrm{eff}}(\phi)$ for the dependence through $c(\phi)$,
\begin{equation}
a_{\mathrm{align}}
=
a_{\mathrm{eff}}(\phi)\big|_{\phi\leq\theta},
\qquad
a_{\mathrm{drift}}
=
-\,a_{\mathrm{eff}}(\phi)\big|_{\phi=\pi,\;f=f'_{\min}} ,
\label{eq:authority_link}
\end{equation}
so $a_{\mathrm{align}}$ is the outward authority available once reorientation
completes and $a_{\mathrm{drift}}$ the inward authority the same set produces
at the lowest reachable thrust while misalignment is maximal. The certificate
charges the second over the reorientation interval and credits the first
afterwards.

\subsection{Escape Barrier and Analytic Inversion}
\label{subsec:barrier}

An escape has two phases: reorientation, then aligned braking. We charge the
clearance consumed during reorientation before computing braking distance. By
Assumption~\ref{asm:reorient},
\begin{equation}
t_{\mathrm{rot}}(\phi)
=
t_{\mathrm{lag}}
+
\frac{\max\{0,\phi-\theta\}}{\omega_{\max}},
\tau_{\mathrm{eff}}
=
\tau_d+t_{\mathrm{rot}}(\phi),
\label{eq:reorientation_time}
\end{equation}
where $\tau_d$ denotes the control delay. During $\tau_{\mathrm{eff}}$, no
effective braking is credited.

This bookkeeping needs one check. The identified response of
Sec.~\ref{subsec:setup}, $\psi(t)\approx\omega_{\max}\max\{0,t-t_{\mathrm{lag}}\}$,
makes a rotation by $\phi$ take $t_{\mathrm{lag}}+\phi/\omega_{\max}$, whereas
Assumption~\ref{asm:reorient} charges only
$t_{\mathrm{lag}}+\max\{0,\phi-\theta\}/\omega_{\max}$; the shortfall
$\theta/\omega_{\max}$ is paid for by the control delay.

\begin{lemma}[Reorientation budget]
\label{lem:budget}
Suppose $\theta\leq\omega_{\max}(\Delta t-t_{\mathrm{lag}})$ and
$\tau_d\geq\Delta t$. Then for every $\phi\in[0,\pi]$,
\begin{equation}
\tau_{\mathrm{eff}}(\phi)
\;\geq\;
t_{\mathrm{lag}}+\frac{\phi}{\omega_{\max}},
\label{eq:budget}
\end{equation}
i.e.\ the interval over which the certificate credits no braking is at least
the true time needed to align the thrust axis.
\end{lemma}

\begin{proof}
If $\phi\leq\theta$, then $\tau_{\mathrm{eff}}=\tau_d+t_{\mathrm{lag}}$ and
Eq.~\eqref{eq:budget} reduces to $\tau_d\geq\phi/\omega_{\max}$, which holds
because
$\phi\leq\theta\leq\omega_{\max}(\Delta t-t_{\mathrm{lag}})<\omega_{\max}\Delta t
\leq\omega_{\max}\tau_d$. If $\phi>\theta$, then
$\tau_{\mathrm{eff}}=\tau_d+t_{\mathrm{lag}}+(\phi-\theta)/\omega_{\max}$ and
Eq.~\eqref{eq:budget} reduces to $\tau_d\geq\theta/\omega_{\max}$, which holds
because $\theta/\omega_{\max}\leq\Delta t-t_{\mathrm{lag}}\leq\Delta t\leq\tau_d$.
\end{proof}

Both are design conditions, not vehicle assumptions, and both hold for the
plant of Sec.~\ref{subsec:setup}:
$\theta=0.141\leq\omega_{\max}(\Delta t-t_{\mathrm{lag}})=0.161~\mathrm{rad}$
and $\tau_d=\Delta t$. Where they fail,
$\tau_{\mathrm{eff}}=\tau_d+t_{\mathrm{lag}}+\phi/\omega_{\max}$ is
unconditionally sound.

Once the thrust direction is aligned with the escape direction, the
available outward acceleration is
\begin{equation}
a_{\mathrm{align}}(\mathbf{x},u^-)
=
\frac{f'_{\max}}{m}
-
g\,\mathbf{n}^{\top}(\mathbf{p})\mathbf{e}_3
-
a_{\mathrm{dist}}^{\mathrm{bound}}(\mathbf{x}).
\label{eq:aligned_authority}
\end{equation}
We also upper-bound the inward drift accumulated before braking begins by
\begin{equation}
a_{\mathrm{drift}}(\mathbf{x},u^-)
=
\frac{f'_{\min}}{m}
+
g\big(\mathbf{n}^{\top}(\mathbf{p})\mathbf{e}_3\big)_+
+
a_{\mathrm{dist}}^{\mathrm{bound}}(\mathbf{x}).
\label{eq:drift_accel}
\end{equation}
The required escape distance is then
\begin{equation}
d_{\mathrm{esc}}(\mathbf{x},u^-)=\underbrace{v_{\mathrm{cl}}\tau_{\mathrm{eff}}+\tfrac{1}{2}a_{\mathrm{drift}}\tau_{\mathrm{eff}}^2}_{\text{no braking available}}
+\underbrace{\frac{\big(v_{\mathrm{cl}}+a_{\mathrm{drift}}\tau_{\mathrm{eff}}\big)^2}{2a_{\mathrm{align}}}}_{\text{braking}},
\label{eq:escape_distance}
\end{equation}
defined on the domain
$\mathcal{X}_{\mathrm{esc}}
=
\left\{
(\mathbf{x},u^-):
a_{\mathrm{align}}(\mathbf{x},u^-)>0
\right\}$.

With $d_{\mathrm{res}}=d_{\mathrm{free}}-d_0$ the clearance left after the
nominal margin $d_0$, the escape-aware barrier is
\begin{equation}
h_{\mathrm{esc}}(\mathbf{x},u^-)
=
d_{\mathrm{res}}(\mathbf{p})
-
d_{\mathrm{esc}}(\mathbf{x},u^-),
(\mathbf{x},u^-)\in\mathcal{X}_{\mathrm{esc}}.
\label{eq:escape_barrier}
\end{equation}
so $h_{\mathrm{esc}}\geq 0$ means the remaining clearance covers both
reorientation and braking. Since $t_{\mathrm{rot}}$, $a_{\mathrm{drift}}$, and
$a_{\mathrm{align}}$ are independent of $v_{\mathrm{cl}}$,
Eq.~\eqref{eq:escape_barrier} is quadratic in the closing speed; setting
$h_{\mathrm{esc}}=0$ and writing $A:=a_{\mathrm{align}}$,
$D:=a_{\mathrm{drift}}$, $\tau:=\tau_{\mathrm{eff}}$ gives
\begin{equation}
v_{\mathrm{cl}}^{\max}(d_{\mathrm{res}},\phi)
=
\sqrt{A}\sqrt{(A+D)\tau^2+2d_{\mathrm{res}}}
-
(A+D)\tau .
\label{eq:vcl_max}
\end{equation}
For $\tau\to 0$ and $D=0$ this reduces to the classical bound
$\sqrt{2A d_{\mathrm{res}}}$: the barrier recovers the standard braking
certificate when reorientation is instantaneous, and depends explicitly on
$\phi$ otherwise.

\subsection{Safety Conditions and Conservatism}
\label{subsec:conservatism}

Eq.~\eqref{eq:escape_distance} is meaningful only when positive outward
acceleration is available after alignment, so we separate an authority
condition from a distance condition:
$h_{\mathrm{auth}}(\mathbf{x},u^-)
:=
a_{\mathrm{align}}(\mathbf{x},u^-)$,
and
$h_{\mathrm{dist}}(\mathbf{x},u^-)
:=
d_{\mathrm{res}}(\mathbf{p})
-
d_{\mathrm{esc}}(\mathbf{x},u^-)$.

The augmented state $(\mathbf{x},u^-)$ is certified as escape-recoverable
when
\begin{equation}
h_{\mathrm{auth}}(\mathbf{x},u^-)>0,
h_{\mathrm{dist}}(\mathbf{x},u^-)\geq 0.
\label{eq:escape_conditions}
\end{equation}
The first ensures braking is possible once aligned; the second that the
remaining clearance covers reorientation and braking. The first is inactive
over most of the envelope: by Eq.~\eqref{eq:aligned_authority} it fails only
when $f'_{\max}/m<g\,\mathbf{n}^{\top}\mathbf{e}_3+a_{\mathrm{dist}}^{\mathrm{bound}}$,
i.e.\ when the escape direction points steeply upward, so it becomes active for
ceilings and overhangs. It is retained because $d_{\mathrm{esc}}$ is undefined
once $a_{\mathrm{align}}\leq 0$: it keeps the distance condition well posed
rather than adding a safety requirement.

Let $\mathcal{E}$ be the set of augmented states from which some admissible
input trajectory drives $v_{\mathrm{cl}}$ to zero before $d_{\mathrm{res}}$ is
consumed, for all $\mathbf{w}\in\mathcal{W}$, with admissibility understood
through Assumption~\ref{asm:reorient}.

\begin{theorem}[Conservatism direction]
\label{thm:conservatism}
Under Assumption~\ref{asm:reorient}, on
$\mathcal{X}_{\mathrm{esc}}=\{(\mathbf{x},u^-):
a_{\mathrm{align}}(\mathbf{x},u^-)>0\}$, we have
\begin{equation}
\left\{
(\mathbf{x},u^-):
h_{\mathrm{dist}}(\mathbf{x},u^-)\geq 0
\right\}
\subseteq
\mathcal{E}.
\label{eq:soundness}
\end{equation}
Thus, the certified set does not overestimate escapability.
\end{theorem}

\begin{proof}
Split the maneuver at $\tau_{\mathrm{eff}}$. On $[0,\tau_{\mathrm{eff}})$ the
certificate credits no outward braking, and with $\mathbf{n}$ fixed the worst
inward growth of the closing speed is bounded by $a_{\mathrm{drift}}$, so
$v_{\mathrm{cl}}(t)\leq v_{\mathrm{cl}}+a_{\mathrm{drift}}t$; the clearance
consumed before braking is therefore at most
$v_{\mathrm{cl}}\tau_{\mathrm{eff}}+\tfrac12 a_{\mathrm{drift}}\tau_{\mathrm{eff}}^2$
and the speed at the start of braking at most
$v_{\mathrm{cl}}+a_{\mathrm{drift}}\tau_{\mathrm{eff}}$. By
Assumption~\ref{asm:reorient} with Lemma~\ref{lem:budget} the thrust is aligned
by $\tau_{\mathrm{eff}}$, and on $\mathcal{X}_{\mathrm{esc}}$ we have
$a_{\mathrm{align}}>0$, so the remaining speed is dissipated within
$(2a_{\mathrm{align}})^{-1}(v_{\mathrm{cl}}+a_{\mathrm{drift}}\tau_{\mathrm{eff}})^2$.
The two phases require exactly $d_{\mathrm{esc}}$ in
\eqref{eq:escape_distance}, so $h_{\mathrm{dist}}\geq 0$ gives
$d_{\mathrm{res}}\geq d_{\mathrm{esc}}$ and $(\mathbf{x},u^-)\in\mathcal{E}$.
\end{proof}

The witnessing maneuver is explicit: the \emph{reference escape policy}
$\pi_{\mathrm{esc}}$ slews the thrust axis toward $\mathbf{n}$ at
$\omega_{\max}$ while ramping the magnitude down to $f'_{\min}$, then switches
to $f'_{\max}$ once aligned. Theorem~\ref{thm:conservatism} certifies that
\emph{this} deliberately suboptimal maneuver completes within
$d_{\mathrm{res}}$. The certificate is conservative by construction --- no
braking credited during reorientation, the disturbance through its worst inward
projection, $\mathbf{n}$ held fixed. One term is charged in the opposite
direction, and we bound it.

\begin{remark}[Thrust-ramp transient]
\label{rem:transient}
$a_{\mathrm{drift}}$ is evaluated at the window endpoint $f'_{\min}$, but
$\pi_{\mathrm{esc}}$ needs up to one interval to ramp down from $f^-$. On that
sub-interval Eq.~\eqref{eq:drift_accel} is not an upper bound. The excess is
bounded: the additional inward speed is at most
\begin{equation}
\Delta v
=
\frac{\dot f_{\max}\Delta t^{2}}{2m},
\label{eq:ramp_dv}
\end{equation}
and the additional clearance it consumes is at most
$\Delta v\big[\tau_{\mathrm{eff}}+(2v_e+\Delta v)/(2a_{\mathrm{align}})\big]$
with $v_e:=v_{\mathrm{cl}}+a_{\mathrm{drift}}\tau_{\mathrm{eff}}$. At the
parameters of Sec.~\ref{subsec:setup} this gives $\Delta v=0.6~\mathrm{m/s}$
with $\tau_{\mathrm{eff}}=0.62~\mathrm{s}$ at $\phi=90^\circ$ and
$1.13~\mathrm{s}$ at $\phi=\pi$, so for the closing speeds of
Sec.~\ref{subsec:closed_loop} the clearance it consumes is a few tenths of a
metre: larger than $d_0$, but well inside the margins the controller actually
holds (Table~\ref{tab:results}). The term vanishes when $f^-=f'_{\min}$, and
replacing $f'_{\min}$ by $f^-$ removes the caveat unconditionally at the cost
of a uniformly smaller certified set. This is the one place where
Theorem~\ref{thm:conservatism} is exact only up to an explicitly bounded
term.
\end{remark}

The result is a pointwise escape certificate, not a recursive-feasibility or
forward-invariance guarantee: Eq.~\eqref{eq:escape_conditions} certifies that
an escape exists from the current augmented state.

\subsection{MPC Embedding}
\label{subsec:mpc}
The conditions depend on $u^-$, which fits predictive control naturally, the
decision variables of an MPC already forming an input sequence. At step $k$ we
set $u^-_{k|k}=u_{k-1}$, and for later nodes
\begin{equation}
u^-_{k+\ell|k}=u_{k+\ell-1|k},
\ell=1,\ldots,N.
\label{eq:previous_input_chain}
\end{equation}
Thus, the dependence on $u^-$ can be propagated along the horizon without
augmenting the system state.

We impose the authority and distance conditions as discrete-time CBF
constraints along the prediction horizon:
{\small
\begin{align}
h_{\mathrm{auth}}(\mathbf{x}_{k+\ell+1|k},u^-_{k+\ell+1|k})
&\geq
(1-\alpha_1)
h_{\mathrm{auth}}(\mathbf{x}_{k+\ell|k},u^-_{k+\ell|k}),
\label{eq:cbf_auth}
\\
h_{\mathrm{dist}}(\mathbf{x}_{k+\ell+1|k},u^-_{k+\ell+1|k})
&\geq
(1-\alpha_2)
h_{\mathrm{dist}}(\mathbf{x}_{k+\ell|k},u^-_{k+\ell|k}).
\label{eq:cbf_dist}
\end{align}
}
for $\ell=0,\ldots,N-1$, where $\alpha_1,\alpha_2\in(0,1]$ set how quickly the
two margins may decrease.

The resulting finite-horizon optimal control problem is
\begin{align}
\min_{\{u_{k+\ell|k}\}_{\ell=0}^{N-1}}
\quad
&
\sum_{\ell=0}^{N-1}
\left(
\|\mathbf{x}_{k+\ell|k}-\mathbf{x}^{\mathrm{ref}}_{k+\ell|k}\|_Q^2
+
\|u_{k+\ell|k}\|_R^2
\right)
\nonumber
\\
\text{s.t.}\quad & \text{\eqref{eq:quad_dynamics}, \eqref{eq:one_step_input_set}, \eqref{eq:previous_input_chain},}\nonumber\\&\ \ \text{\eqref{eq:cbf_auth}, \eqref{eq:cbf_dist}},\quad \ell=0,\dots,N-1.
\label{eq:mpc}
\end{align}

We freeze the obstacle normal along the horizon,
$\mathbf{n}_{k+\ell|k}=\mathbf{n}_{k|k}$, which keeps the escape quantities
closed form and is exact at the first node. The resulting constraints are
pointwise escape certificates at the predicted nodes, not a
recursive-feasibility guarantee.

\section{Experiments}
\label{sec:experiments}

\subsection{Experimental Setup}
\label{subsec:setup}

Closed-loop experiments run on a plant more detailed than the model used to
compute the certificate: a $13$-state quadrotor with first-order motor lag and
a cascaded attitude loop. The certificate is evaluated on
Eq.~\eqref{eq:quad_dynamics} from position, velocity, and realized thrust
direction.

The plant is integrated at a $2~\mathrm{ms}$ inner step; the MPC runs at
$\Delta t=0.10~\mathrm{s}$ over $N=20$ nodes with control delay
$\tau_d=0.10~\mathrm{s}$ and decay parameters
$(\alpha_1,\alpha_2)=(0.10,0.25)$. Thrust is constrained to
$[f_{\min},f_{\max}]=[2,22]~\mathrm{N}$ with
$\dot f_{\max}=120~\mathrm{N/s}$. Obstacles are inflated by
$d_{\mathrm{safe}}=0.35~\mathrm{m}$, so $d_{\mathrm{free}}$ is the signed
distance to the inflated boundary and
$d_{\mathrm{res}}=d_{\mathrm{free}}-d_0$, with $d_0=0.15~\mathrm{m}$, is what
the escape maneuver may consume. The gust bound is
$\|\mathbf{w}\|\leq 1.5~\mathrm{m/s^2}$.

The mass is $m=1.0~\mathrm{kg}$, so $a_{\max}=f_{\max}/m=22~\mathrm{m/s^2}$.
The reorientation parameters are identified, not assumed: thrust-direction step
commands over a grid of turning amplitudes, azimuths, and thrust levels are
fitted to $\psi(t)\approx\omega_{\max}\max\{0,t-t_{\mathrm{lag}}\}$, and over
$108$ responses this gives $\omega_{\max}=3.05~\mathrm{rad/s}$ and
$t_{\mathrm{lag}}=0.047~\mathrm{s}$. The worst single-interval turn under
full-slew commands is $0.156~\mathrm{rad}$, and $\theta$ is set to $90\%$ of
it, $\theta=0.141~\mathrm{rad}$ ($8.1^\circ$), the only tuning applied.

Scenarios are generated deterministically from a seed, so all controllers face
the same initial condition, obstacle motion, and disturbance sequence.
Continuous metrics use the Wilcoxon signed-rank test, binary outcomes the exact
McNemar test, all Holm--Bonferroni corrected. An episode counts as a success if
the vehicle reaches the goal region without contacting an obstacle within the
time limit; contact and timeout are reported separately where they differ.

\subsection{Baselines and Escape Referee}
\label{subsec:baselines}

All controllers share the same MPC implementation, tasks, constraints, and
disturbances, so the comparison isolates the certificate. B1 is a geometric
state-only MPC-CBF~\cite{zeng2021}, representing certification by instantaneous
clearance. B2 is a worst-case constant-margin method~\cite{mandaokar2026}
inflating the required clearance from global bounds on delay, braking, and
input variation, testing whether a fixed margin can replace a state-dependent
certificate.

B3 is a rate-feasible backup CBF flown
online~\cite{gurriet2018,chen2021,singletary2022}, and is the baseline that
matters most: it certifies exactly the quantity our closed form approximates,
integrating at each step the same backup maneuver the referee uses, and differs
from ours only in \emph{where} that condition is enforced. It shares the plant,
scenarios, seeds, disturbances, and backup maneuver of every other controller,
and is retuned for nothing.

An offline referee evaluates logged states, rolling out the same backup
maneuver and marking a state escapable if it recovers before the boundary. It
is an independent conservative audit, not a controller baseline. It defines the metric used throughout: writing
$\mathcal{C}_k$ for the event that a controller's own barrier certifies the
logged state at step $k$, and $\mathcal{R}_k$ for the event that the referee's
rollout recovers from it, the certified-but-inescapable (CBI) count is
\begin{equation}
\mathrm{CBI}=\textstyle\sum_k \mathbf{1}[\,\mathcal{C}_k \wedge \neg\mathcal{R}_k\,],
\label{eq:cbi}
\end{equation}
the steps a controller declared safe and the oracle could not escape. It is a
soundness counter, not a performance one: a collision the barrier never
certified contributes nothing. Because it commits to one fixed maneuver it is conservative, so
$\mathrm{CBI}=0$ gives the stronger inclusion
$\{\text{certified}\}\subseteq\{\text{referee-escapable}\}\subseteq\mathcal{E}$.

\begin{figure}[!tb]
\centering
\includegraphics[width=0.48\textwidth]{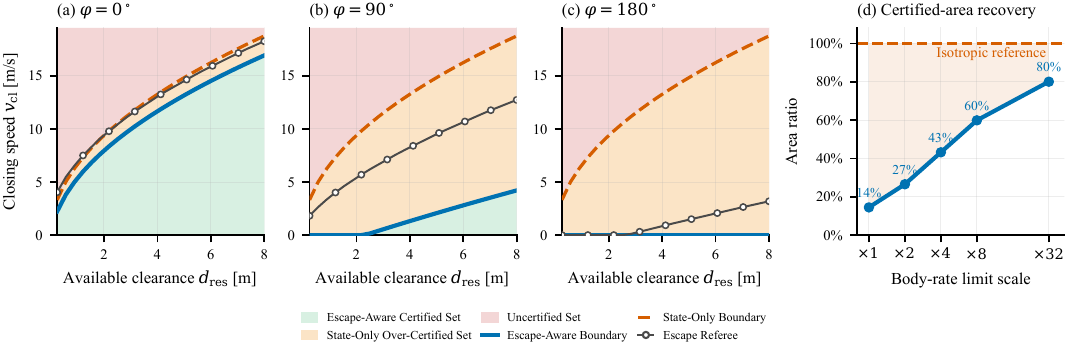}
\caption{Certified safe regions under thrust-direction misalignment.
(a)--(c) $\phi=0^\circ,90^\circ,180^\circ$: green is the escape-aware certified
set, orange the state-only over-certified set, gray the rollout referee. (d)
The worst-case certified area approaches the isotropic reference as the
body-rate limit grows.}
\label{fig:certified_sets}
\end{figure}

\subsection{Certified Sets under Thrust-Direction Misalignment}
\label{subsec:certified}

Fig.~\ref{fig:certified_sets} compares the state-only boundary, ours, and the
rollout referee in the $(d_{\mathrm{res}},v_{\mathrm{cl}})$ plane. At
$d_{\mathrm{res}}=4$~m the state-only barrier certifies $13.26$~m/s at every
misalignment; ours certifies $11.66$, $1.32$, and ${\approx}0$~m/s at
$\phi=0,90,180^\circ$.

The referee supplies the independent reference, and the comparison is the
quantitative form of Corollary~\ref{cor:gap_phi}. Its boundary falls from
$13.05$ to $8.26$ to $0.84$~m/s over the same angles: the state-only barrier is
essentially exact when the thrust already points outward ($1.02\times$ at
$\phi=0^\circ$) but over-certifies by $1.61\times$ at $90^\circ$ and
$15.8\times$ at $180^\circ$. The error is directional, not a constant offset,
and no direction-independent margin absorbs it. Our closed form errs in the
opposite direction, recovering $89\%$, $16\%$, and $0\%$ of the referee's
certified speed. The two errors are not commensurable --- over-certification
admits a state no input can escape, under-certification refuses an escapable
one at the cost measured in Sec.~\ref{subsec:closed_loop} --- so a certificate
is useful only when its error is one-sided: Theorem~\ref{thm:conservatism}
fixes the side, $\mathrm{CBI}=0$ confirms it. Correspondingly,
Fig.~\ref{fig:certified_sets}(d) shows the $\phi=180^\circ$ certified area
rising from $14.5\%$ to $80.0\%$ of the isotropic reference as the body-rate
limit scales $\times1$ to $\times32$; it saturates below $100\%$ because
$\tau_{\mathrm{eff}}$ retains $\tau_d+t_{\mathrm{lag}}=0.147$~s, which does not
scale with $\omega_{\max}$.

\begin{table}[t]
\centering
\caption{Closed-loop performance across three obstacle-avoidance scenarios.}
\label{tab:results}
\tiny
\setlength{\tabcolsep}{1.6pt}
\renewcommand{\arraystretch}{1.0}
\begin{tabular}{@{}llcccccccc@{}}
\toprule
Scenario & Controller & Runs & Success
& Min. clear. [m] & Path [m] & Time [s] & Solve [ms]
& CBI & $\Delta\phi$ [deg] \\
\midrule
S1
& \textbf{Ours} & 30 & 100\% & 3.26 [2.89, 3.51] & 39.6 & 8.2 & 46.6 & 0 & 83.2 \\
& Ours$_{\ell=0}$\textsuperscript{$\ddagger$} & 30 & 57\%\textsuperscript{**} & 0.21 [$-0.30$, 0.39]\textsuperscript{***} & 24.9 & 9.0 & 35.4 & 9 & 9.4 \\
& B1 & 30 & 93\% & 2.16 [2.03, 2.36]\textsuperscript{***} & 32.5 & 6.3 & 51.7 & 0 & 55.3 \\
& B2\textsuperscript{$\dagger$} & 30 & 93\% & 4.78 [4.49, 4.96] & 40.4 & 8.0 & 34.8 & 0 & 91.6 \\
& B3 & 30 & 93\% & 0.80 [0.57, 1.27]\textsuperscript{***} & 29.3 & 5.8 & 36.3 & 0 & 9.0 \\
\midrule
S2
& \textbf{Ours} & 30 & 100\% & 1.78 [1.69, 1.87] & 33.5 & 5.8 & 36.7 & 0 & 34.4 \\
& Ours$_{\ell=0}$ & 30 & 97\% & 0.74 [0.62, 0.85]\textsuperscript{***} & 44.8 & 10.3 & 27.5 & 0 & 22.4 \\
& B1 & 30 & 97\% & 1.46 [1.40, 1.53]\textsuperscript{***} & 32.8 & 6.8 & 39.2 & 0 & 31.6 \\
& B2\textsuperscript{$\dagger$} & 30 & 0\%\textsuperscript{***} & $-0.48$ [$-0.64$, $-0.27$]\textsuperscript{***} & 4.7 & $-$ & 1213.0 & 0 & $-$ \\
& B3 & 30 & 80\% & 0.10 [0.02, 0.24]\textsuperscript{***} & 28.4 & 6.6 & 43.2 & 0 & 2.2 \\
\midrule
S3
& \textbf{Ours} & 50 & 100\% & 2.03 [1.86, 2.20] & 38.8 & 8.5 & 42.6 & 0 & 77.6 \\
& Ours$_{\ell=0}$ & 50 & 100\% & 1.38 [1.27, 1.53]\textsuperscript{***} & 40.4 & 8.9 & 23.5 & 0 & 66.3 \\
& B1 & 50 & 100\% & 1.37 [1.24, 1.50]\textsuperscript{***} & 37.1 & 8.2 & 29.8 & 0 & 69.7 \\
& B2 & 50 & 100\% & 1.51 [1.38, 1.76]\textsuperscript{***} & 36.0 & 8.8 & 29.8 & 0 & 77.6 \\
& B3 & 50 & 100\% & 0.67 [0.59, 0.74]\textsuperscript{***} & 35.1 & 8.1 & 27.8 & 0 & 48.9 \\
\bottomrule
\end{tabular}
\vspace{0.2em}
\begin{minipage}{\linewidth}
\scriptsize
Median [IQR] over paired seeds; markers on baseline rows denote comparison
against ours after correction. CBI counts controller-certified but
referee-unsafe steps; $\Delta\phi$ is the directional margin (larger is safer).
Success-rate differences against B3 are not significant at these sample sizes;
clearance, band-occupancy and $\Delta\phi$ differences are.
Ours$_{\ell=0}$ is ours with \eqref{eq:cbf_auth}--\eqref{eq:cbf_dist} imposed
only between nodes $0$ and $1$; every other setting is unchanged.
\textsuperscript{$\dagger$}More than half infeasible steps; excluded from
$\Delta\phi$. \textsuperscript{$\ddagger$}$8/30$ episodes largely infeasible
(Sec.~\ref{subsec:ablation}). \textsuperscript{**}$p<0.01$,
\textsuperscript{***}$p<0.001$.
\end{minipage}
\end{table}

\begin{figure}[htbp]
\centering
\includegraphics[width=\linewidth]{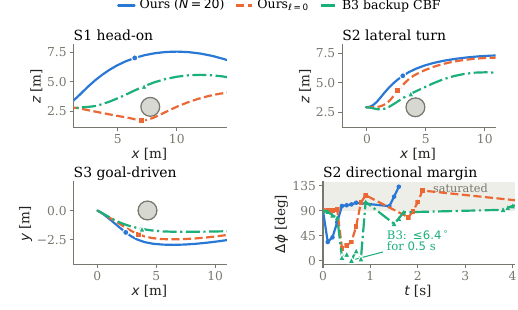}
\vspace{-5mm}
\caption{One episode per scenario (seed $1$) around the closest approach. The disc is
the obstacle at that instant; markers are each controller's own closest
approach. Panels share a window size and are projected on the plane where
avoidance occurs ($x$--$z$ for S1 and S2, $x$--$y$ for S3). Bottom right:
directional margin over the S2 encounter --- B3 spends half a second at or
below $6.4^\circ$, ours never falls below $33.8^\circ$; shading marks samples
where escape survives at any misalignment.}
\label{fig:traj_all}
\vspace{-5mm}
\end{figure}

\subsection{Closed-Loop Behavior}
\label{subsec:closed_loop}
\begin{figure}[htbp]
  \centering
  \begin{subfigure}[b]{0.32\columnwidth}
    \includegraphics[width=\linewidth]{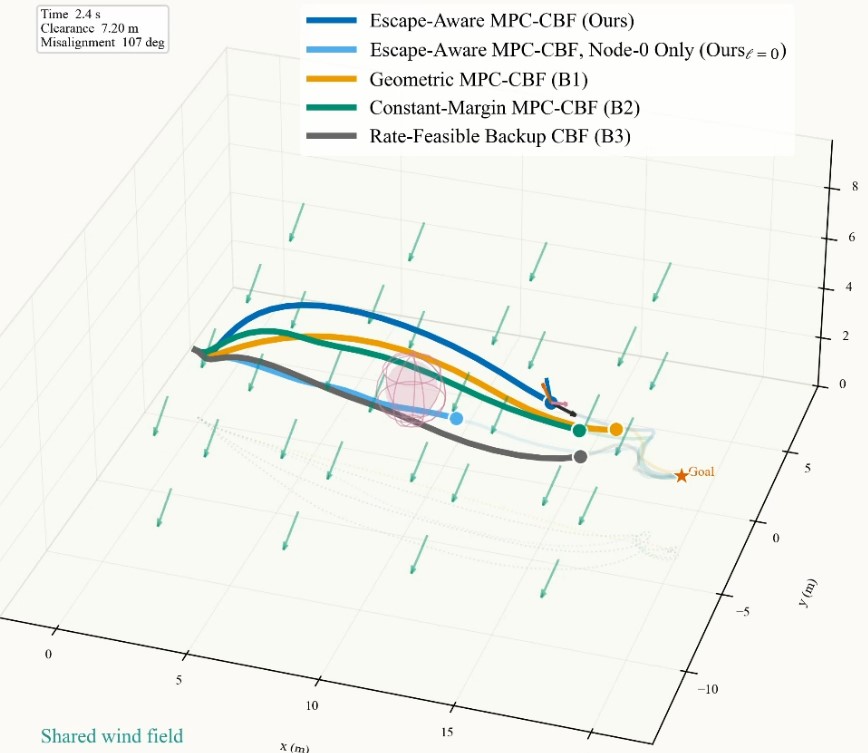}
    \caption{}\label{fig:traj_s1}
  \end{subfigure}\hfill
  \begin{subfigure}[b]{0.32\columnwidth}
    \includegraphics[width=\linewidth]{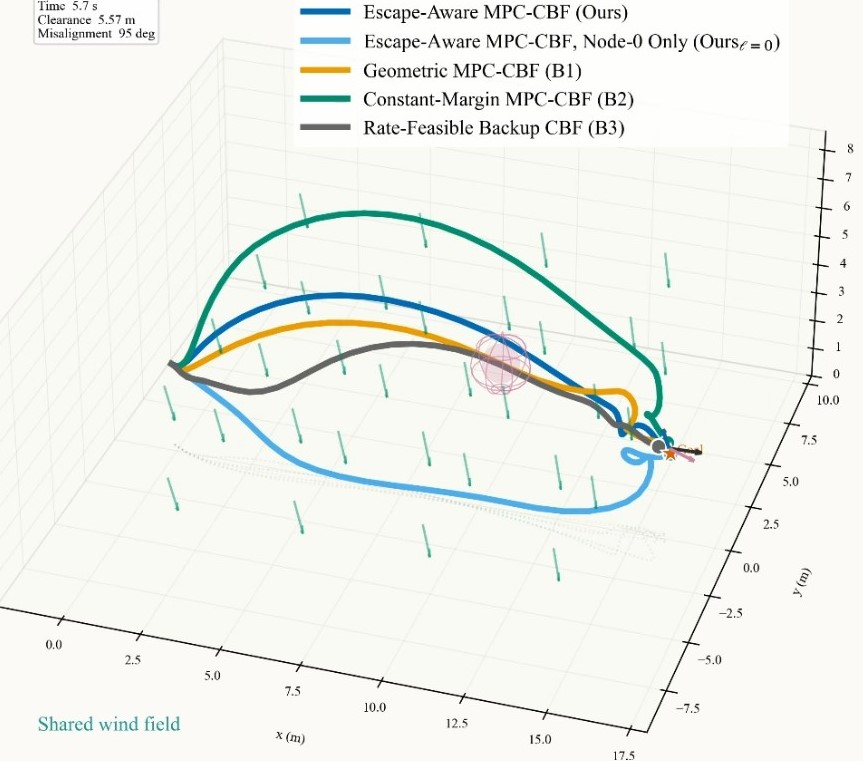}
    \caption{}\label{fig:traj_s2}
  \end{subfigure}\hfill
  \begin{subfigure}[b]{0.32\columnwidth}
    \includegraphics[width=\linewidth]{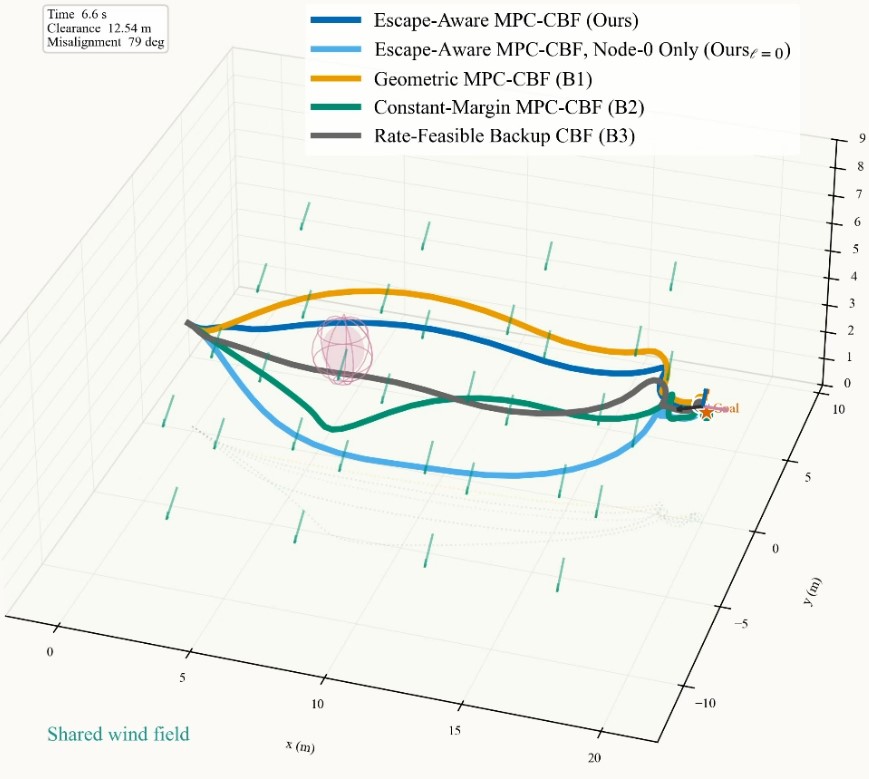}
    \caption{}\label{fig:traj_s3}
  \end{subfigure}
  \vspace{-3mm}
  \caption{Trajectories under the time-varying wind field.
  (a) S1: head-on encounter. (b) S2: turning obstacle.
  (c) S3: goal beyond the obstacle.}
  \label{fig:traj}
\end{figure}

S1 is a head-on encounter with a constant-velocity obstacle, matching
Proposition~\ref{prop:gap}; S2 circumnavigates a turning obstacle; S3 places
the goal beyond the obstacle, so goal tracking alone drives the vehicle toward
the state-only gap (Fig.~\ref{fig:traj}). All scenarios run under a strong
wind field whose magnitude and direction vary continuously over time. The
controller has no access to the instantaneous wind vector and uses only the
bound $\|\mathbf{w}\|\leq 1.5~\mathrm{m/s^2}$ of Sec.~\ref{subsec:setup}.

Table~\ref{tab:results} reports the outcome. Our controller completes every
episode, which B1 does not ($93\%$ in S1, $97\%$ in S2), and holds a larger
minimum clearance than B1 throughout ($3.26$ vs
$2.16$~m, $1.78$ vs $1.46$~m, $2.03$ vs $1.37$~m, all significant after
correction) and a larger directional margin ($83.2^\circ$ vs $55.3^\circ$,
$34.4^\circ$ vs $31.6^\circ$, $77.6^\circ$ vs $69.7^\circ$). It never enters
the Proposition~\ref{prop:gap} band, where B1 spends $7.5\%$ of its closing
steps in S1, $6.5\%$ in S3, and none in S2. The band is evaluated at $\phi=\pi$ on geometry
alone, so by Corollary~\ref{cor:gap_phi} occupancy is necessary but not
sufficient for being certified yet inescapable --- it measures how often a
state-only barrier \emph{can} be wrong, and S1 and S3 are also the scenarios in
which $h_I$ itself fails in closed loop (Sec.~\ref{subsec:ablation}).

The margin is not free, but the cost is narrower than the certificate itself
suggests. Against B1 we travel $22\%$ farther and take $30\%$ longer in S1;
in S2 the paths are within $2\%$ and we in fact reach the goal sooner
($5.8$ vs $6.8$~s), and in S3 both gaps are under $5\%$. What remains is in
the nonlinear program, not the algebraic certificate: our median solve time is
$10\%$ and $6\%$ \emph{below} B1's in S1 and S2 and $43\%$ above it in S3,
and the largest median here, $46.6$~ms, sits well inside the $100$~ms control
period.

B3 carries the contribution, enforcing the same escape condition and differing
only in \emph{where}. Its median clearance is $0.80$, $0.10$, and $0.67$~m
against our $3.26$, $1.78$, and $2.03$ ($p_{\mathrm{adj}}<10^{-7}$), and its
directional margin collapses to $9.0^\circ$, $2.2^\circ$, and $48.9^\circ$,
reaching $0^\circ$ in the worst S1 and S2 episodes (Fig.~\ref{fig:traj_all}).
It occupies the gap band more than B1, reaches $37$ steps the referee marks
inescapable, and collides in $6.7\%$ of S1 and $20\%$ of S2 episodes; at $30$
paired seeds those rates are descriptive, while the clearance, band, and
$\Delta\phi$ differences are significant. The mechanism is not cost: B3 solves in $0.65$--$1.18\times$ our median time,
cheaper than us in S1 and S3 and \emph{dearer} in S2 --- the scenario where it
collides most --- and its condition is the stronger one; what it cannot do
is enforce it anywhere but the state already reached, by which point the thrust
is misaligned and the clearance committed.

B2 needs separate reading: its globally tuned margin gives the largest S1
clearance ($4.78$~m) yet still collides in $6.7\%$ of episodes, and in S2 it is
infeasible from the \emph{first} step, describing a controller that never acts.
No single direction-independent margin serves all three; adapting it online is
the remedy of~\cite{liu2026adaptive}, and our result says what it must adapt
to.

The CBI count is zero for every controller except the node-$0$ ablation, and
reading it correctly matters. For B1 the zero holds even inside the gap band:
those excursions occurred at sub-worst-case misalignment, so over-certification
showed up as eroded margin rather than as an inescapable state. For B3 the same
zero coexists with $37$ inescapable steps, because B3's certificate \emph{is}
the rollout --- it never certifies such a state and is simply left with nothing
admissible. Those are feasibility, not soundness, failures, and they are the
point.

\begin{table}[htbp]
\centering
\caption{Internal ablations, $20$ paired seeds, locus held at all $N$ nodes.
Rows~3--5 switch one modeling choice of our own controller; row~2 collapses the
certificate to the state-only barrier $h_I$. Everything else is unchanged.}
\label{tab:ablation}
\vspace{-2mm}
\footnotesize
\setlength{\tabcolsep}{3.5pt}
\renewcommand{\arraystretch}{1.05}
\begin{tabular}{@{}lccccccc@{}}
\toprule
& \multicolumn{3}{c}{Min. clearance [m]} & & \multicolumn{3}{c}{Success [\%]} \\
\cmidrule(r){2-4}\cmidrule(l){6-8}
Variant & S1 & S2 & S3 & & S1 & S2 & S3 \\
\midrule
Ours, unmodified            & $3.21$ & $1.77$ & $2.05$ & & $100$ & $100$ & $100$ \\
$h_I$ at all $N$ nodes$^{\S}$ & $2.02$ & $0.51$ & $1.58$ & & $85$ & $100$ & $75$ \\
No directional reachability & $2.58$ & $1.30$ & $1.33$ & & $100$ & $100$ & $100$ \\
Constant one-step authority & $-0.43$ & $-0.35$ & $1.55$ & & $0$ & $0$ & $0$ \\
Worst-case drift charged    & $1.97$ & $-0.26$ & $1.56$ & & $55$ & $35$ & $100$ \\
\bottomrule
\multicolumn{8}{@{}p{\columnwidth}@{}}{\scriptsize $^{\S}$Eq.~\eqref{eq:escape_barrier}
with $\theta\!\to\!\pi$, $\tau_d\!=\!t_{\rm lag}\!=\!0$, $\mathcal{W}\!=\!\{0\}$,
$f'_{\max}\!=\!f_{\max}$, which collapses $h_{\rm esc}$ to $h_I$ of
Eq.~\eqref{eq:state_only_barrier}: the barrier of Prop.~\ref{prop:gap},
imposed at every node rather than none.}\\
\end{tabular}
\vspace{-4mm}
\end{table}

\subsection{Ablation Study}
\label{subsec:ablation}

A certificate is fixed by \emph{what} it computes and \emph{where} it is
imposed; B1 and B3 change both. Table~\ref{tab:ablation} varies the first at a
fixed locus of all $N$ nodes.

\textbf{Content.} Row~2 collapses Eq.~\eqref{eq:escape_barrier} to $h_I$ --- four
terms at once, the comparison Proposition~\ref{prop:gap} predicts rather than a
one-term ablation --- and is granted our own locus, not the single node at which
a stopping barrier is normally applied. It still fails $15\%$ of S1 and $25\%$
of S3 while completing all of S2, exactly the scenarios in which B1 enters the
Proposition~\ref{prop:gap} band ($7.5\%$, $6.5\%$ of closing steps, none in
S2), and completes S2 on $0.51$~m against our $1.77$. Row~3 restores every term
but the direction cap and recovers completion at a cost of $20$--$35\%$ of the
clearance: the reorientation and drift terms keep the vehicle feasible, the cap
keeps the margin. Row~4 freezes $\mathcal{A}_1(u^-)$ through braking and
completes no episode anywhere, stalling where it cannot crash ($-0.43$,
$-0.35$~m in S1--S2; $1.55$~m and no completion in S3) --- a certificate made
unsatisfiable rather than unsafe. Row~5 runs against us: charging the transient
of Remark~\ref{rem:transient} unconditionally, the variant needing no caveat in
Theorem~\ref{thm:conservatism}, forfeits $45\%$ and $65\%$ of S1 and S2
episodes. That $h_I$ fails where the weaker barrier of B1 does not is the
mechanism, not a contradiction: imposing $h_{k+1}\geq(1-\alpha)h_k$ on
$d_{\mathrm{free}}$ caps the closing speed near
$\alpha d_{\mathrm{free}}/\Delta t$, below the
$\sqrt{2a_{\max}d_{\mathrm{free}}}$ that $h_I$ licenses over the $1$--$4$~m
clearances here. A geometric barrier is conservative by accident; $h_I$
authorizes exactly the fast approach in which reorientation time binds.

\textbf{Locus.} Holding certificate, horizon, cost, weights, decay rates, and
solver fixed, we impose Eqs.\eqref{eq:cbf_auth}--\eqref{eq:cbf_dist} only between
nodes $0$ and $1$ (Table~\ref{tab:results}, Ours$_{\ell=0}$): clearance falls
from $3.26$, $1.82$, $2.03$~m to $0.21$, $0.74$, $1.38$ against B3's $0.80$,
$0.10$, $0.67$, and the S1 margin from $83.2^\circ$ to $9.4^\circ$ against B3's
$9.0^\circ$. The same condition loses the margin when checked only at the state
already reached which is the effect B3 entangles. Excluding the eight fallback episodes
leaves $0.31$~m and $16.4^\circ$, and it is the only variant with nonzero CBI
($9$ steps, all in colliding episodes): no violation of
Theorem~\ref{thm:conservatism}, the referee being conservative in that
direction, but only this locus reaches where certificate and referee disagree.
The converse cell is closed by structure : a rollout has no analytic derivative,
so with $186$ variables $20$ such constraints cost
$\mathcal{O}(10^{3})$ integrations per Jacobian at every interior-point
iteration, against B3's one rollout per step.

\section{Conclusion}
\label{sec:conclusion}

We studied quadrotor safety certification under body-rate-limited thrust
reorientation, showing that state-only input-constrained barriers over-certify
because they ignore the previously applied thrust direction, and that the gap
has a closed form vanishing as the attainable body rate grows. We addressed it
with an escape-aware certificate on $(\mathbf{x},u^-)$ that evaluates one-step
reachable thrust authority, charges reorientation time, and admits a
closed-form barrier and an analytic inverse for the maximum certifiable closing
speed.

On a 13-state quadrotor the controller completes every tested scenario while
holding larger clearance and directional margin than a geometric MPC-CBF, stays
feasible where a constant margin is infeasible from the first step, and is
never certified-but-inescapable under an independent conservative backup-CBF rollout audit. Against that same rollout flown as a controller it holds significantly
larger clearance and margin at no more than $2.2\times$ the solve time, which
isolates the contribution: both enforce the same escape condition, and the
closed form is what lets it be enforced across a horizon rather than at one
state. Ablations confirm that the stopping barrier $h_I$, enforced over that
same horizon, still fails $15$--$25\%$ of episodes in the two scenarios that
occupy the gap band, that directional reachability is essential, and that a
constant-authority model is unusable.

A key limitation is the lack of hardware validation. The soundness of the proposed certificate depends on Assumption~\ref{asm:reorient}, which reduces the attitude dynamics to a single effective slew rate. This assumption is identified and validated only on the simulated plant considered in this work~\cite{shi2025wholebody}; its validity on physical hardware remains to be established.

\def\IEEEbibitemsep{-0.3ex plus .2pt}
\bibliographystyle{IEEEtran}
\bibliography{references}

\end{document}